\documentclass[11pt]{llncs}
\usepackage[letterpaper,hmargin=1in,vmargin=1.25in]{geometry}
\usepackage{cancel}
\usepackage[table]{xcolor}

\usepackage{soul}

\usepackage{amssymb,amsmath,url}
\usepackage{graphicx}

\usepackage{cite}

\usepackage{tikz}
\usetikzlibrary{positioning, shapes}

\usepackage{framed}
\usepackage{wrapfig}
\usepackage{epstopdf}

\usepackage{multirow}

\newcommand{\bbR}{\mathbb{R}}

\newcommand{\bfT}{\mathbf{T}}

\newcommand{\KeyGen}{{\sf KeyGen}}
\newcommand{\Enc}{{\sf Enc}}
\newcommand{\Dec}{{\sf Dec}}

\newcommand{\sm}{\rm \scriptscriptstyle}

\newcommand{\nc}{L}
\newcommand{\nep}{{n_{\rm \scriptscriptstyle epoch}}}
\newcommand{\up}{\sm uploadCtxt}
\newcommand{\down}{\sm downloadCtxt}

\usepackage{ifthen}

\begin{document}

%
%
%
%
%
\title{Privacy-Preserving Deep Learning via  Weight Transmission\thanks{This is the full version of  a conference paper \cite{NSS_Phong17}. }}
\author{\sc Le Trieu Phong$^{(a)}$ and Tran Thi Phuong$^{(b)}$}
\institute{$^a$National Institute of Information and Communications Technology (NICT), Tokyo, Japan\\ {\tt phong@nict.go.jp}\\[2ex] $^{b}$Faculty of Mathematics and Statistics, Ton Duc Thang University, Ho Chi Minh City, Vietnam\\{\tt tranthiphuong@tdtu.edu.vn}}


\maketitle



\begin{abstract}
This paper considers the scenario that multiple data owners wish to apply a machine learning method over the combined dataset of all owners to obtain the best possible learning output but do not want to share the local datasets owing to privacy concerns.  We design systems for the scenario that the stochastic gradient descent (SGD) algorithm is used as the machine learning method because SGD (or its variants) is at the heart of recent deep learning techniques over neural networks. Our systems differ from existing systems in the following features: {\bf (1)} any activation function can be used, meaning that no privacy-preserving-friendly approximation is required; {\bf (2)} gradients computed by SGD are not shared but  the weight parameters are shared instead; and {\bf (3)} robustness against colluding parties even in the extreme case that only one honest party exists. We prove that our systems, while privacy-preserving, achieve the same learning accuracy as SGD and hence retain  the merit of deep learning with respect to accuracy. Finally, we  conduct several experiments using benchmark datasets, and show that our systems  outperform previous system in terms of learning accuracies.\\[2ex]
{\bf Keywords:} Privacy preservation,  stochastic gradient descent, distributed trainers, neural networks.
\end{abstract}
%

\section{Introduction}

\subsection{Background}
Stochastic gradient descent (SGD) and its variants are  important methods in machine learning. In recent years, these methods have become vital tools in deep learning based on neural networks, producing surprisingly high learning accuracy.

While crucial for SGD and machine learning in general, accumulating a massive collection of data in one central location is not easy for several reasons including the issue of data privacy. In such cases, it is more desirable to keep the data in its original location while making use of it. 

In this paper we assume that there are $L$ distributed datasets owned  by $L$ trainers. Each trainer aims to apply SGD to all $L$ datasets to maximize the learning accuracy; while not wishing to share its own dataset  to minimize the risk of data leakage. 

Activation functions inside neural networks are  crucial for making the networks learn effectively. Depending on the specific dataset, one activation function may empirically work better than  others. Therefore, it is important to be able to choose the activation function when training neural networks. Ideally, when designing protocols for privacy-preserving SGD and its variants, there is no restriction such as privacy-preserving friendliness imposed on activation functions.

Placing no restrictions on activation functions, Shokri and Shmatikov \cite{SS15}  presented a system for privacy-preserving asynchronous SGD in which an {\em honest-but-curious} server is employed to hold the  gradients computed by SGD using local data.  The system of Shokri and Shmatikov  is designed to share only a part of local gradients with the server, and their experiments show that it is sufficient to obtain high accuracy for real datasets such as the MNIST dataset containing images of handwritten digits \cite{MNIST}.

Shokri and Shmatikov \cite{SS15} did not use cryptographic tools such as secret sharing or homomorphic encryption in their system. As the minimum to protect the communication between trainers and the server, standard Transport Layer Security (TLS) is used. As a result, they wrote: \lq\lq {\it Our system achieves this [privacy goal] at a much lower performance cost than cryptographic techniques such as secure multi-party computation or homomorphic encryption and is suitable for deployment in modern large-scale deep learning.}" We follow their footsteps, aiming to use only TLS when designing our privacy-preserving systems, as shown in Table \ref{table_1}. 

\subsection{Our contributions}\label{our_contri}

\begin{table*}[t]
\centering
\caption{Privacy-preserving distributed deep learning systems.}\label{table_1}
\scalebox{0.8}{
\begin{tabular}{|c|c|c|c|c|c|c|}
\hline
\bf Paper &\bf Use of cryptography &\bf Activation  & \bf Security against  & \bf Security against & \textcolor{black}{\bf Trainer}  \\
& & \bf function & \bf \textcolor{black}{honest-but-}curious server & \bf  collusion & \textcolor{black}{\bf transmission}\\
\hline
\cite{SS15} & Transport Layer Security (TLS) & Any & No & No &\textcolor{black}{(Parts of) Gradients}\\
\cite{Phong_IEEE_18} & Additively homomorphic encryption \& TLS& Any & Yes & No &\textcolor{black}{(Encrypted) Gradients}\\
\hline
This paper & Transport Layer Security (TLS) & Any & Yes & Yes & \textcolor{black}{Weights}\\
\hline
\end{tabular}
}
\end{table*}

We  propose two novel systems for privacy-preserving SGD to protect the local data of all trainers. Similar to the system in \cite{SS15} and its enhanced variant \cite{Phong_IEEE_18}, our systems can handle any activation function. However, very different from the systems in \cite{SS15,Phong_IEEE_18}, in our systems the trainers do not share gradients but share the weights of the neural network. Our systems are called the SNT system and FNT system, depending on the connection with Server-aided Network Topology (SNT) or Fully-connected Network Topology (FNT)  among the trainers. The SNT and FNT systems are described in Section \ref{our_system_section}, and have the following properties of security and accuracy.  In addition, a comparison with the systems in \cite{SS15,Phong_IEEE_18} is given in Table \ref{table_1}.

\medskip
\noindent
{\bf $\bullet$ Security against an \textcolor{black}{honest-but-}curious server (in SNT).} {\it Our SNT system leaks no information on both the trained weight parameters and the data of participants to an honest-but-curious server.} See Theorem \ref{sec_thm}.

\medskip\noindent
{\bf $\bullet$ Security against extreme collusion (both SNT and FNT).} {\it Even in the extreme case that only one trainer (wlog., called trainer 1) is honest, and the other trainers and the server are \textcolor{black}{honest-but-curious}, it is difficult to   recover the local data of  honest trainer 1.} See Theorems \ref{collusion_thm} and  \ref{collusion_thm_rnt}. 

As a special case, this property also ensures that a possibly \textcolor{black}{honest-but-}curious trainer cannot recover the data of the others. In addition, it is easy to combine results \cite{AbadiCGMMT016} on differentially private SGD with our systems as noted in Section \ref{enhancements_subsection}, so that local data items can be protected in the sense of  differential privacy.

\medskip\noindent
{\bf $\bullet$ Accuracy (both SNT and FNT).} {\it Our systems achieve identical accuracy to SGD trained over the joint dataset of all participants.} See Theorems \ref{acc_thm} and  \ref{acc_thm_rnt}.


\medskip
\noindent
{\bf Experiments with UCI datasets and image datasets (MNIST, CIFAR-10, CIFAR-100).} In Section \ref{exp_res}, we conduct a number of experiments using several benchmark datasets to demonstrate the general applicability of our systems. Specifically, using UCI datasets, we show that the accuracy of our systems outperforms the previously reported results in \cite{AonoHPW16} for the task of privacy-preserving classification. 

Using a multilayer perceptron (MLP) and a convolutional neural network (CNN) over the MNIST dataset of 50,000 images of handwritten digits, the running time of our system is less than three times that of the original SGD while maintaining  learning accuracies. 

{\color{black}
Using deeper neural networks such as ResNet over the CIFAR-10 and CIFAR-100 datasets, we show that the overheads of  cryptographic operations and transmissions are very little compared to the training on plain data. In addition, the learning accuracies are kept  complying with known results in the literature. }

%

\subsection{Technical overview}\label{tech_overview}
{\noindent \bf Stochastic gradient descent (SGD), no privacy protection.} In SGD,  the weight parameters (flattened to a vector) $W$ for the neural network are initialized randomly. At each iteration, gradient $G$ is computed by using a data item and the current $W$, and then the weight vector is updated as follows:
\begin{eqnarray}
W := W- \alpha \cdot G \label{asgd_update}
\end{eqnarray}
where $\alpha$ is a scalar called the learning rate. The updating process   is repeated until a desired minimum for a predefined cost function based on cross-entropy or squared error is reached.

{\medskip \noindent \bf Shokri-Shmatikov systems.} In the system of \cite{SS15}[Sect. 5] with one central server and many distributed trainers, the server stores a weight vector $W_{\rm server}$. The update rule in (\ref{asgd_update}) is modified as follows: 
\begin{eqnarray}
W_{\rm server} := W_{\rm server} - \alpha \cdot G_{\rm local}^{\rm selective}\label{asgd_SS15}
\end{eqnarray}
in which vector $G_{\rm local}^{\rm selective}$ contains a selected (say $1\%\sim10\%$) part of the locally computed gradients.  The update using (\ref{asgd_SS15}) allows each participant to choose which gradients to share globally with the hope of reducing the risk of leaking sensitive information on the participant's local dataset to the \textcolor{black}{honest-but-}curious server. However, as shown in \cite{Phong_IEEE_18}, such selected gradients may result in leaking  information on the local data.

In \cite{SS15}[Sect. 7], Shokri and Shmatikov showed an additional technique of using differential privacy as a  countermeasure against indirect leakage from gradients. Their strategy was to add Laplace noises{\color{black}\footnote{\color{black}In fact, it turns out that the parameter $\epsilon$ of $\epsilon$-differential privacy in \cite{SS15} can be as large as 600,000, guaranteeing no meaningful privacy as showed in \cite{Papernot17}.}} to $G_{\rm local}^{\rm selective}$ at (\ref{asgd_SS15}). However, because secrecy and differential privacy are orthogonal, still there is potential for leakage. 

To solve the problem of data leakage from gradients, Phong et al. \cite{Phong_IEEE_18} proposed that one can use additively homomorphic encryption to protect  gradients from the \textcolor{black}{honest-but-}curious server. However, even with homomorphic encryption, the system in \cite{Phong_IEEE_18} is still weak against the collusion of an \textcolor{black}{honest-but-}curious server and some \textcolor{black}{honest-but-}curious trainers. Indeed, from (\ref{asgd_SS15}), if the colluding parties know the current and previous weight parameters, then they can compute the selective gradients of other (honest) trainers.

{\medskip \noindent \bf Our systems.} Our approach is very different from those of \cite{SS15,Phong_IEEE_18}, because trainers do not share the gradients (either plain or encrypted) but instead share the {\em weights}. In our SNT system in Section \ref{our_system_section}, we make use of the following weight update process: each trainer $i$  takes the current $W_{\rm server}$ from the server and initially sets its weight vector as
\begin{eqnarray}
W^{(i)} := W_{\rm server}\label{WiWg}
\end{eqnarray}
 then the trainer {\em repeatedly} performs the local training for $W^{(i)}$ using its local dataset via SGD,
\begin{eqnarray}
W^{(i)}:= W^{(i)} -\alpha \cdot G_{\rm local} \label{updateWi}
\end{eqnarray}
and uploads $W^{(i)}$ to the server to replace $W_{\rm server}$, namely 
\begin{eqnarray}
W_{\rm server} := W^{(i)}.\label{WgWi}
\end{eqnarray}
The trained weights may be proprietary and should be kept secret from those who are not trainers (e.g., the server in our setting). To this end, the trainers will use a shared symmetric key (kept secret from the server) to encrypt $W^{(i)}$ and $W_{\rm server}$ using (\ref{WiWg}) and (\ref{WgWi}). The local update in (\ref{updateWi})  by each trainer can be performed  as is after decryption because the trainer holds the decryption key. Thus, the \textcolor{black}{honest-but-}curious server will only handle ciphertexts without knowing the secret key and hence it will obtain no information on the datasets. The process in (\ref{WiWg}), (\ref{updateWi}), and (\ref{WgWi}) can be succinctly described as the process of updating the weight vector  $W_{\rm server}$ using the data of  trainer $i$. If all trainers repeat the process, then $W_{\rm server}$ is trained over the combined datasets of all trainers as desired. This final $W_{\rm server}$ is known to all trainers owing to the shared symmetric  key.

Because of the repeated application of (\ref{updateWi}), sharing weights is intuitively equivalent to sharing a weighted sum of all gradients. Therefore, from a privacy-preserving perspective, shared weights are more robust against information leakage. Indeed, we prove in Theorems \ref{collusion_thm} and \ref{collusion_thm_rnt} that inverting local data from the weights can be seen as  the  problem of solving a system of nonlinear equations, where the number of equations is smaller than the number of variables (data items). 


\subsection{Difference with the conference version}\label{diff_NSS2017}
An abridged version of this paper was in \cite{NSS_Phong17}.  This version improves \cite{NSS_Phong17} in the following ways: ({\bf 1}) we  add the FNT system and compare with the SNT system; ({\bf 2}) we state and prove  that our systems are relatively robust against the extreme collusion of trainers in Theorems \ref{collusion_thm}  and \ref{collusion_thm_rnt};  ({\bf 3}) \textcolor{black}{we consider additional hedges with respect to the weight privacy for our systems in Section \ref{add_hedge}, in which the techniques of dropouts and no overfitting are extensively used in later experiments};  ({\bf 4}) \textcolor{black}{to demonstrate the wide applicability of our systems} we add the experiments on UCI datasets, showing that the accuracies of our systems are better than those of \cite{AonoHPW16,codaspy_AonoHPW16} regarding the task of classification; {\color{black} ({\bf 5}) we add the experiments on the CIFAR-10 and CIFAR-100 datasets, showing that cryptographic operations and transmissions in our systems incur very little  overhead compared to training over plain data when  neural networks become deeper. }

\section{Other related works}
{\color{black} Abadi et al. \cite{AbadiCGMMT016} (modifying  and extending \cite{SongCS13, BassilyST14}) proposed  differentially private SGD to protect  the trained weight parameters when the data is centralized, which  is  orthogonal but complementary to this work. Indeed, each trainer in our system can locally use the techniques of \cite{AbadiCGMMT016,SongCS13, BassilyST14}  as discussed in Section \ref{add_hedge}.}

Gilad{-}Bachrach et al. \cite{Gilad-BachrachD16} presented a system called {\it CryptoNets}, which allows homomorphically encrypted data to be feed forward to an already-trained neural network. Because {\it CryptoNets} assumes that the weights in the neural network have been trained beforehand, the system aims to make predictions for individual data items. The goals of our paper and \cite{SS15,Phong_IEEE_18} differ from that of \cite{Gilad-BachrachD16}, as our systems and that of Shokri and Shmatikov exactly aim to train the weights utilizing multiple data sources, while {\it CryptoNets} \cite{Gilad-BachrachD16} does not. In the same vein of research, the works \cite{Liu_ONN_2017, Riazi_ASIACCS_2018,Rouhani:2018:DSP,JuvekarVC18} examined secure neural network prediction.

Hitaj et al. \cite{HitajAP17}, employing generative adversarial networks, showed that sharing gradients are dangerous in collaborative learning, thus emphasizing the observation in \cite{Phong_IEEE_18}. Chang et al. \cite{OCY017} also independently proposed a system similar to our FNT system, without any  security analysis.

Aono et al. \cite{AonoHPW16,codaspy_AonoHPW16} used a polynomial approximation for logistic regression so that they could design a privacy-preserving system for classification tasks. Owing to the polynomial approximation, the classification  accuracy was inferior to that of the original logistic regression especially when the data dimension increased. Hesamifard et al. \cite{HesamifardTGW18} used polynomial approximation for continuous activation functions.

Using a two-server model, Mohassel and Zhang \cite{MohasselZ17} proposed protocols for privacy-preserving linear regression, logistic regression, and an MLP in which secure-computation-friendly activation functions were employed. For example, they used an approximation for the logistic function. There were no experimental results on CNNs reported in \cite{MohasselZ17} perhaps owing to the hurdles of implementing the Fast Fourier Transform in the case of secure computation. Subsequently, Mohassel and Rindal also considered a three-server model in \cite{MR2018}.

{\color{black}Using secret sharing, Bonawitz et al. \cite{BIK17} proposed a secure aggregation method and applied it to deep neural networks to aggregate user-provided model updates. In \cite{BIK17} the trained weight parameters appear in plaintext on a central server which may be undesirable in many scenarios \cite{TramerZJRR16}, while in our systems this is not  the case. The communication cost in our systems can be  $O(1)$ (namely, independent of the number of trainers) per weight update, while it  is $O(N_{\rm trainer})$ in \cite{BIK17}, where $N_{\rm trainer}$ is the number of trainers. }

Transfer learning is a method in machine learning where a pretrained weight  of a neural network over a dataset is reused in another neural network over a different dataset. For example, a pretrained weight  for classifying cars can be loaded into a related but different neural network for the task of classifying buses. The reasoning in this example is that the pretrained weight  have been optimized to \lq\lq understand" the components of cars, so should be useful for learning and classifying buses. Our systems share the principle  of reusing weight  in transfer learning, but have the  difference that  all the neural networks of the trainers are identical for the same task. 

\section{Preliminaries}\label{dl_premi}
We recall a few preliminaries on cryptography and machine learning in this section.

{\medskip \noindent \bf Symmetric encryption.} Symmetric encryption schemes consist of the following (possibly probabilistic) poly-time algorithms: $\KeyGen(1^\lambda)$ takes a security parameter $\lambda$ and generates the secret key $K$;  $\Enc(K,m)$, equivalently written as $\Enc_K(m)$, produces $c$ which is the ciphertext of message $m$; $\Dec(K, c)$ returns message $m$ encrypted in $c$.

\medskip
\noindent
Ciphertext indistinguishability against chosen plaintext attacks \cite{Goldreich2004} (or CPA security for short) ensures that no  information is leaked from ciphertexts. We will use symmetric encryption as provided by TLS.

{\medskip \noindent \bf Neural networks.}  Figure \ref{nn_example} shows a neural network with 5 inputs, 2 hidden layers, and 2 outputs. The node with $+1$ represents the bias term. The neuron (including the bias\footnote{It is possible to  exclude the bias nodes and use a separate variable $b$ (see, for example, \cite{SDLT}).}) nodes are connected via weight variables $W$. In a deep-learning structure of a neural network, there can be multiple layers each with thousands of neurons. Each neuron node (except the bias node) is associated with an {\it activation function} $f$. 
Typical examples of $f$  are 
\begin{eqnarray*}
{\rm [ReLU]}\qquad  f(z) &=& \max\{0,z\}   \\
{\rm [leaky\ ReLU]} \qquad f_{\beta}(z) &=& \left\{ \begin{array}{ll}z & \quad {\rm if }\  z>0 \\ \beta z & \quad {\rm if }\  z\leq 0 \end{array} \right.  \\
{\rm [sigmoid]}\qquad  f(z) &=& \frac{e^z}{e^z+1} \\
{\rm [tanh]} \qquad f(z) &=& \frac{e^z - e^{-z}}{e^z + e^{-z}}.
\end{eqnarray*}
The nonlinearity of these activation functions is important for the network to learn complicated  data distributions. 

\begin{figure}[t]
\centering
\includegraphics[scale=0.4]{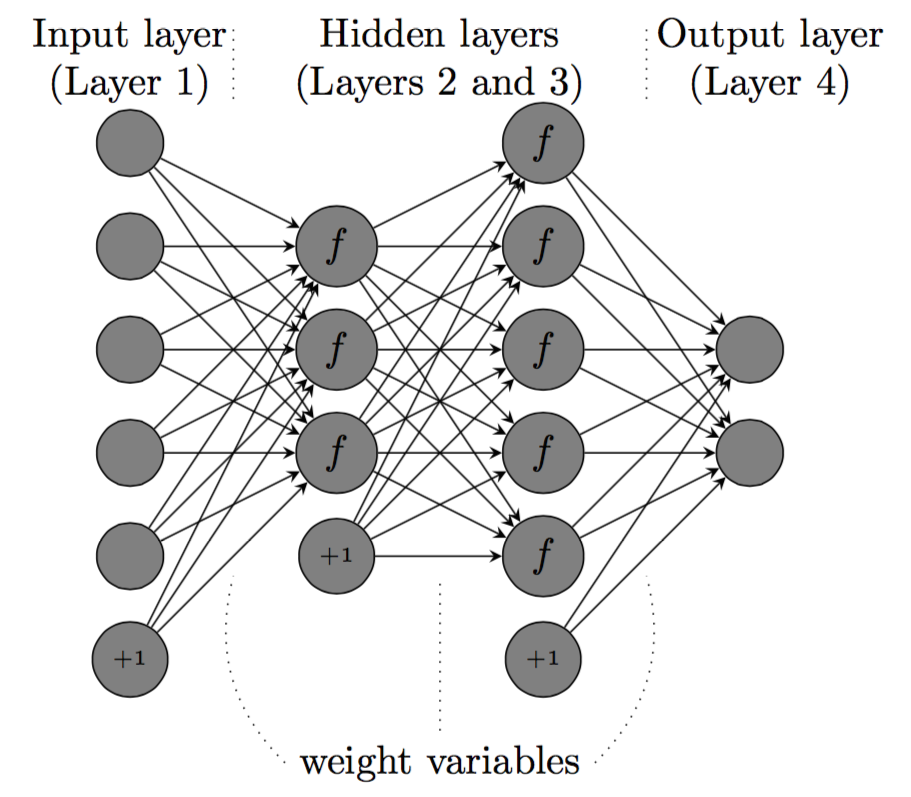}
\caption{Neural network with  activation function $f$.}\label{nn_example}
\end{figure}

The learning task is, given a training dataset, to determine these weight variables to minimize a predefined cost function such as the cross-entropy or the squared error cost function \cite{SDLT}. The cost function defined over one data item $data =(x,y)$ with input $x$ and truth value $y$ will be written  as $J(W,x,y)$ in which $W$ is the weight parameters. 

{\medskip \noindent \bf Stochastic gradient descent (SGD) (see also \cite{SDLT}).} Let $W$ be the weight parameters and  let $G\leftarrow  \frac{\delta J(W,x,y)}{\delta W}$ be the corresponding gradients of the cost function $J(W,x,y)$ with respect to the variables in $W$. The variable update rule in  SGD is as follows for a learning rate $\alpha (t)\in \bbR$ at time $t$: $$W \leftarrow W - \alpha(t)\cdot G$$ in which $\alpha(t)\cdot G$ is component-wise multiplication. 
The learning rate $\alpha(t)$ can be a constant or can be changed adaptively.

To make SGD work more efficiently, a typical technique is to use vectorization via a mini-batch of data instead of a single data item. Namely, multiple data items (e.g., 256) are packed correspondingly into matrices $(X,Y)$ and then the following processes are carried out:
\begin{eqnarray}
{\rm (compute\ the\ gradient)} \ \ G &\leftarrow& \frac{\delta J(W,X,Y)}{\delta W} \label{G_minibatch}\\
{\rm (update\ the\ weight)} \ \ W &\leftarrow& W - \alpha(t)\cdot G \label{W_minibatch}
\end{eqnarray}
where $\alpha(t)$ is the learning rate at time $t$.

\section{Our proposed systems: privacy-preserving SGD via weight \textcolor{black}{transmission}}\label{our_system_section}
We propose two systems for privacy-preseving SGD: one with a Server-aided Network Topology (SNT) and the other with Fully-connected Network Topology (FNT) in which each connection is via a separate TLS/SSL channel. The SNT system makes use of an honest-but-curious server (e.g., cloud server) while the FNT system does not. One trade-off is given in Table \ref{trade_off_SNT_RNT}. The SNT system may be suitable when the number of trainers $L$ is large (say $\geq 20$), while the FNT system may be suitable for relatively small $L$ (say $\leq 20$).

\textcolor{black}{In both SNT and FNT, we  assume that the datasets of the trainers have the same column (feature) format. In other words, the datasets are horizontally distributed.}

\begin{table}[h]
\caption{Properties of the SNT and FNT systems.}\label{trade_off_SNT_RNT}
\centering
\scalebox{0.9}{
\begin{tabular}{|c|c|c|c|}
\hline
\bf System &\bf Central server needed? & \bf Number of TLS connections\\
\hline
Our SNT system &Yes & $L$\\
Our FNT system & No& $L(L-1)/2$\\
\hline
\end{tabular}}
\end{table}

\subsection{Our Server-aided Network Topology (SNT) system} The system is depicted in Figure \ref{star_system}. There is one common server and multiple distributed trainers. The server is assumed to be {\em honest-but-curious}: it is honest in its operation but curious regarding data.  Each trainer is connected  with the server via a separate communication channel such as a TLS/SSL channel. The symmetric key $K$ is shared between trainers and is kept secret to the server. Notationally, $Enc_K(\cdot)$ is a symmetric encryption with  key $K$, which  protects the weight vector against the \textcolor{black}{honest-but-}curious server. The training data of the trainers is used only locally, while the encrypted weight vector $Enc_K(W)$ is sent back and forth between the server and the trainers.  The weight vector $W$ needs to be initialized once, which can be carried out by one of the trainers by simply initializing  $W$ randomly and sending $Enc_K(W)$ to the server. \textcolor{black}{The server either can schedule the trainer order, or can do that randomly: after receiving the weight from a trainer $i$, the server uniformly and randomly chooses a trainer $j\ne i$ to transmit the weight.}

{\color{black}
In Figures \ref{star_system} and \ref{ring_system}, $Dataset_i$ can be either the whole local dataset of trainer $i$, or just a set of a few data batches randomly chosen from the local dataset. If the local datasets are equally distributed, then the former case is preferred; otherwise if the  local datasets are not equally distributed, then the latter case may be preferred for good accuracy. The concrete choice of $Dataset_i$ should depend on the datasets in each specific application.
}
  
\newcommand{\dts}{Dataset}
\begin{figure*}[t]
\begin{tabular}{ccc}
\begin{tabular}{c}
\includegraphics[scale = 0.5]{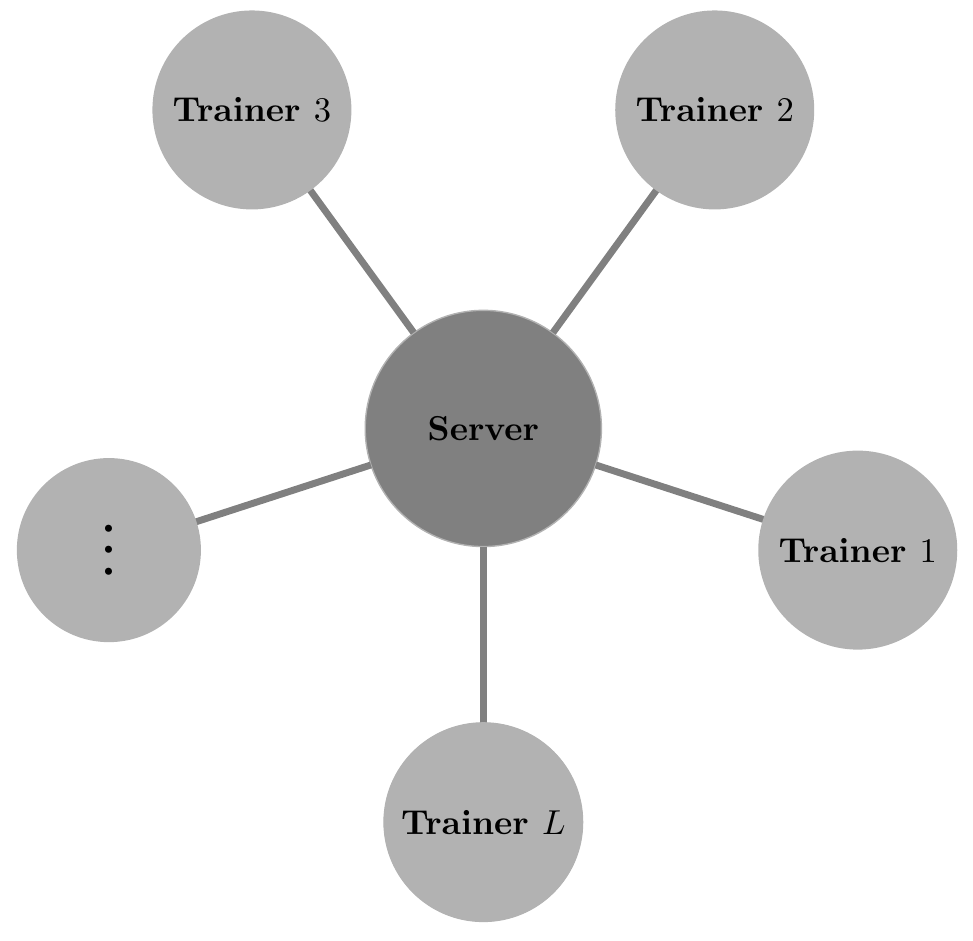}
\end{tabular}
&
\scalebox{0.9}{
\begin{tabular}{|l|}
\hline
Trainer $i$ holding $\dts_i$ \ ($1\leq i \leq L$)\\
\hline
\quad get $Enc_K(W)$ from the server\\
\quad decrypt to obtain  weight $W$\\
\quad for mini-batch $(X,Y)$ in $\dts_i$:\\
\qquad [compute gradient] $G \leftarrow \frac{\delta J(W,X,Y)}{\delta W}$ \\
\qquad [update weight]\ \ \ \ $W \leftarrow W - \alpha(t)\cdot G$ \\
\quad endfor\\
\quad send $Enc_K(W)$ to server\\
\hline
\end{tabular}}
&
\scalebox{0.9}{
\begin{tabular}{|l|}
\hline
Server (honest-but-curious)\\
\hline
\quad global $encW$\\
\quad while True:\\
\qquad get connection from a trainer\\
\qquad if trainer sends $Enc_K(W)$:\\
\quad\qquad $encW\leftarrow Enc_K(W)$\\
\qquad endif\\
\qquad if trainer gets $encW$:\\
\quad\qquad send $encW$ to trainer\\
\qquad endif\\
\hline
\end{tabular}}\\
\bf (a)& \bf (b)&\bf (c)\\
\end{tabular}
\caption{Our proposed SNT system with Server-aided Network Topology where all connections are via secure channels such as TLS: {\bf (a)} many distributed trainers connected with a central server; {\bf (b)}  description of a trainer; {\bf (c)}  description of the central server.}\label{star_system}
\end{figure*}

Each trainer stores the updated weight vector $W$ at each round, then uses a testing dataset to locally check whether or not $W$ has achieved good accuracy \textcolor{black}{(and check overfitting as well)}. Finally, the best $W$ can be shared among the trainers via a separate communication channel if required.

The following theorems establish the properties for the security and accuracy of our systems.

\medskip 

\begin{theorem}[Security against the \textcolor{black}{honest-but-}curious server, SNT system]\label{sec_thm}
In our system in Figure \ref{star_system}, the \textcolor{black}{honest-but-}curious server learns no information on the local datasets of the trainer.
\end{theorem}
\begin{proof}
The \textcolor{black}{honest-but-}curious server passively handles  ciphertexts of a symmetric encryption scheme. Therefore, it obtains no information from the ciphertexts if the symmetric encryption scheme has security against chosen plaintext attacks.
\qed\end{proof}

\begin{theorem}[Security against extreme collusion, SNT system]\label{collusion_thm}
In our system in Figure \ref{star_system}, suppose that only Trainer 1 is honest and the others (Server and Trainers $2, \dots, L$) are malicious. Even with such collusion, the colluding  others cannot recover any data item of  honest Trainer 1 unless they solve a nonlinear equation (or a subset sum problem).
\end{theorem}
\begin{proof}
Referring to Figure \ref{star_system}(b), what the colluding parties (namely, Server and Trainers $2, \dots, L$) obtain from Trainer 1 is the initial weight and  output weight. These weights are linked via a number of gradient descent update steps. Concretely, let the initial weight be $W_0 = W^{(init)}$, the final weight be $W_n = W^{(final)}$, and $(X_i,Y_i)$ ($1\leq i \leq n$)  be the randomly shuffled mini-batches from $Dataset_1$ from the current epoch. Then, the gradient descent update steps are as follows:
\begin{eqnarray*}
G_1 &\leftarrow& \frac{\delta J(W_0,X_1,Y_1)}{\delta W}\\
W_1 &\leftarrow& W_0 - \alpha_1 G_1 \\
&\vdots& \\
G_i &\leftarrow& \frac{\delta J(W_{i-1},X_i,Y_i)}{\delta W}\\
W_i &\leftarrow& W_{i-1} - \alpha_i G_i \\
&\vdots& \\
G_n &\leftarrow& \frac{\delta J(W_{n-1},X_n,Y_n)}{\delta W}\\
W_n &\leftarrow& W_{n-1} - \alpha_n G_n \\
\end{eqnarray*}
in which the scalar $\alpha_i$ ($1\leq i \leq n$) can be equal or different depending on the learning rate schedule of Trainer 1. Reversely, 
\begin{eqnarray}
W^{(final)} &=& W_n \nonumber \\
&=&  W_{n-1} - \alpha_n G_n \nonumber \\
&=& W_{n-2} - \alpha_{n-1} G_{n-1}- \alpha_n G_n \nonumber\\
&\vdots& \nonumber \\
&=& W_0 - (\alpha_1 G_1 + \cdots + \alpha_n G_n) \nonumber\\
&= &W^{(init)} - (\alpha_1 G_1 + \cdots + \alpha_n G_n). \label{subset_sum}
\end{eqnarray}
In short, even though the colluding parties know $W^{(final)}$ and $W^{(init)}$, they can only  compute the weighted sum of gradients 
\begin{eqnarray}
W^{(init)} - W^{(final)}= \alpha_1 G_1 + \cdots + \alpha_n G_n \label{grads_sum}
\end{eqnarray}
 in which $n = |Dataset_1|/ batch\_size \gg 1$ in our expected applications. In (\ref{grads_sum}),  the local dataset and hence each gradient $G_i$ is kept secret, and even $\alpha_i$ may vary secretly depending on the learning rate schedule of Trainer 1. Recovering any data item of $Dataset_1$ from equation (\ref{grads_sum}) is the problem of solving non-linear equation stated in the theorem statement.
 
As seen from (\ref{subset_sum}), releasing neural network weight parameters can be seen as publishing a weighted sum of gradients and the initial weight. Therefore, computing any data item given the weight parameters is also a subset sum problem.
\qed\end{proof}

{\medskip\noindent\bf Remark: data inversion given weights and subset sum problem.}  While inverting data from weight parameters is considered hard implicitly in the literature, to our knowledge, this is the first time such a relation between data secrecy (given weight parameters) and a subset sum problem has been made explicit as in Theorem \ref{collusion_thm}.

\medskip

\begin{theorem}[Accuracy of the SNT system]\label{acc_thm}
Our system in Figure \ref{star_system} functions as running SGD on the combined dataset of all local datasets.
\end{theorem}
\begin{proof}
Our system in Figure \ref{star_system} (when removing all encryption and decryption)  functions as  the left pseudocode in Figure \ref{table_acc_thm}. In addition, the right pseudocode is the equivalent version via setting  ${CombinedDataSet}= \dts_1 \cup \cdots \cup \dts_L$.
\begin{figure*}
\centering
\begin{tabular}{|p{6cm}|p{0.4cm}p{8cm}|}
\hline
{Initialize $W$ randomly}

{for mini-batch $(X,Y)$ in $\dts_1$}:

\qquad $G \leftarrow \frac{\delta J(W,X,Y)}{\delta W}$ 

\qquad $W \leftarrow W - \alpha(t)\cdot G$ 

{endfor}

$\vdots$

{for mini-batch $(X,Y)$ in $\dts_L$}:

\qquad $G \leftarrow \frac{\delta J(W,X,Y)}{\delta W}$ 

\qquad $W \leftarrow W - \alpha(t)\cdot G$ 

{endfor}
&&
{Initialize $W$ randomly}

Set ${CombinedDataSet} = \dts_1 \cup \cdots \cup \dts_L$

{for mini-batch $(X,Y)$ in $CombinedDataSet$}:

\qquad $G \leftarrow \frac{\delta J(W,X,Y)}{\delta W}$ 

\qquad $W \leftarrow W - \alpha(t)\cdot G$ 

{endfor}\\
\hline
\end{tabular}
\caption{Pseudocodes for the proof of Theorem \ref{acc_thm}.}\label{table_acc_thm}
\end{figure*}

In the right pseudocode, the loop ({for ... endfor}) is exactly one epoch of SGD, namely one pass over all data items in ${CombinedDataSet}$. Therefore, our system in Figure \ref{star_system} functions as running SGD as described in Section \ref{dl_premi} on the combined dataset ${CombinedDataSet}$  of all trainers, completing  the proof.
\qed\end{proof}

\begin{figure*}[t]
\centering
\begin{tabular}{cc}
\begin{tabular}{c}
\includegraphics[scale = 0.5]{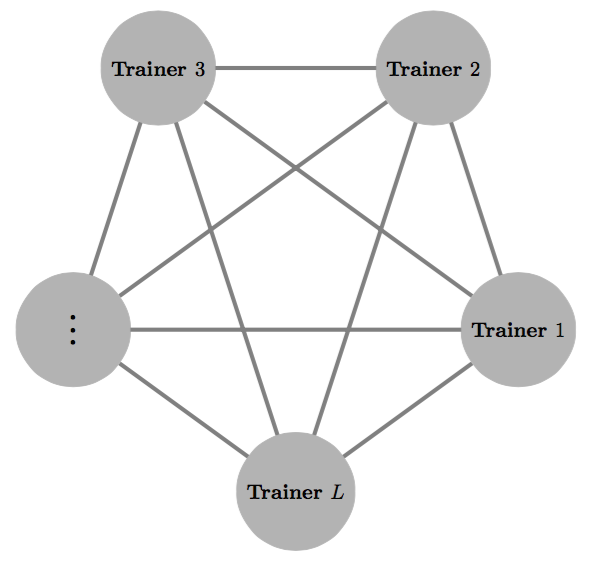}
\end{tabular}
&
\begin{tabular}{|l|}
\hline
Trainer $i$ holding $\dts_i$ \ ($1\leq i \leq L$)\\
\hline
\quad get weight $W$ from the previous trainer\\
\quad for mini-batch $(X,Y)$ in $\dts_i$:\\
\qquad [compute gradient] $G \leftarrow \frac{\delta J(W,X,Y)}{\delta W}$ \\
\qquad [update weight] \ \ \  $W \leftarrow W - \alpha(t)\cdot G$ \\
\quad {\rm endfor}\\
\quad send  weight $W$ to the next trainer\\
\hline
\end{tabular}\\
\bf (a)& \bf (b)\\
\end{tabular}
\caption{Our proposed FNT system with Fully-connected Network Topology where all connections are via secure channels such as TLS: {\bf (a)} many distributed trainers connected with each other; {\bf (b)} description of a trainer.}\label{ring_system}
\end{figure*}

\subsection{Our Fully-connected Network Topology (FNT) system} The system is depicted in Figure \ref{ring_system}. There is no central server; thus there is no need to use symmetric encryption to encrypt the weights as in Figure \ref{star_system}, but the trainers still need secure channels such as TLS to transfer the weights to each other. Because there is no central server, the trainers must agree the order of sending and receiving weight parameters in advance, for example, $1 \to 2 \to \cdots \to L \to 1 \to \cdots$; \textcolor{black}{or simply do transmission in a random order: trainer $i$ takes uniformly and randomly  a number $r$ in $\{1, \dots, L\}\setminus \{i\}$  and sends its weight parameters to trainer $r$}. Each trainer  receives a weight from the previous trainer, and  updates that weight using mini-batches of data from its own dataset, then sends the updated weight to the next trainer. Chang et al. \cite{OCY017} also independently\footnote{To be precise regarding time, the work  \cite{OCY017} appeared after the conference version  \cite{NSS_Phong17}  of this paper.} proposed a system similar to our FNT system but without any security analysis.

\medskip

\begin{theorem}[Security against extreme collusion, FNT system]\label{collusion_thm_rnt}
In our system in Figure \ref{ring_system}, suppose that only Trainer 1 is honest and the others (Trainers $2, \dots, L$) are malicious. Even with such collusion, the (malicious) others cannot recover any data item of the honest Trainer 1 unless they solve a nonlinear equation (or a subset sum problem).
\end{theorem}
\begin{proof}
The proof is identical to that of Theorem \ref{collusion_thm}.
\qed\end{proof}

\begin{theorem}[Accuracy of the FNT system]\label{acc_thm_rnt}
Our system in Figure \ref{ring_system} functions as running SGD on the combined dataset of all local datasets.
\end{theorem}
\begin{proof}
The proof is identical to that of Theorem \ref{acc_thm}.
\qed\end{proof}

\subsection{Additional considerations for our SNT and FNT systems}\label{enhancements_subsection}
Below are some additional considerations for our systems, which can be employed depending on the specific application and its requirements.

{\medskip \noindent \bf Local epochs and central epochs.} An epoch is one pass over all data items in a dataset in the training process. Accordingly, one local epoch is one pass over a local dataset, while one central epoch is one pass over the combined dataset of all trainers. While we consider the number of local epochs as $1$ in the above description of our systems for simplicity, it can be set to any number if required, allowing each trainer to perform multiple passes over its dataset before sending the weights to the next trainer. The training dataset is shuffled randomly before each local epoch. 

Similarly, it is possible that the number of central epochs  is larger than $1$, meaning that there are multiple passes over the trainers. When the number of central epochs  is $k\geq 1$, instead of  equation (\ref{grads_sum}), the following equations are known to the colluding parties when assuming that only Trainer 1 is honest:
\begin{eqnarray}
\begin{aligned}\label{num_eqs}
V_1 &=& \alpha_1 G_{1}^{(1)} + \cdots + \alpha_n G_{n}^{(1)} \\
&\vdots& \\
V_k &=& \alpha_1 G_{1}^{(k)} + \cdots + \alpha_n G_{n}^{(k)} 
\end{aligned}
\end{eqnarray}
where $G_{1}^{(1)} , \dots, G_{n}^{(1)}$, $\dots$, $G_{1}^{(k)} , \dots, G_{n}^{(k)} $ are unknown gradients depending on the mini-batches of data. As the neural network contains nonlinear activation functions, the above can be seen as a system of nonlinear equations of unknown data items. {\color{black}Solving this system of equations is expected to be hard because typically the number of unknowns (i.e. the cardinality of each local dataset) is larger than the number of central epochs $k$ (see Table \ref{un_knowns_andk} for concrete values); and the right-hand side of each equation is non-polynomial. }

\begin{table}[t]
\centering
\caption{\color{black}The number of unknowns and equations in (\ref{num_eqs}). In all cases, \#unknowns $>$ \#equations. If each unknown is a real vector of dimension $d$, then the number of variables in (\ref{num_eqs}) becomes $d\times$ {\#unknowns}.}\label{un_knowns_andk}
\begin{tabular}{>{\color{black}}c >{\color{black}}c  >{\color{black}}cc}
\hline
\bf Dataset & \bf \#unknowns (\ref{num_eqs}) & \bf \#equations in (\ref{num_eqs}) \\
(in Section \ref{exp_res})& ($=$ \#data items in local dataset)& ($=$ \#central epochs $k$)\\
\hline
Pima &$30$ & 20\\
Breast Cancer & 19 & 5\\
Banknote Authentication &39 & 1\\
Adult Income&$1628$ & 120\\
Skin/NonSkin & $9802$ & 20\\
Credit Card Fraud Detection &$12816$&30 \\
\hline
MNIST & 10,000& $639$ (MLP),  193 (CNN)\\
CIFAR-10 &10,000& 50 (CNN), $100$ (ResNet) \\
CIFAR-100 &10,000& 100 (ResNet)\\
\hline
\end{tabular}
\end{table}

{\medskip \noindent \bf Local data augmentation.} Deep learning is often referred to as a \lq\lq data hunger", and it is not hard for a neural network to overfit a training dataset. To deal with the problem of overfitting and hence improve the learning accuracy, each trainer can freely employ any  data augmentation technique in the local training process. For example, a data item such as an image can be rotated, sheared, or flipped (or even mixed up as in \cite{mixup_paper2018}) so that the data volume is significantly increased. 

{\medskip \noindent \bf Using other optimizers.} Instead of SGD, the trainers in our SNT and FNT systems can use other optimizers such as RMSProp \cite{TH12} or Adam \cite{KingmaB14}, because only the weights are sent from each trainer to the next, as in the FNT system, or to the server in the SNT system. In our experiments we use both SGD and Adam as optimizers.

\subsection{\textcolor{black}{Additional hedges for our systems}}\label{add_hedge}
{\color{black}
In the above sections, we have consider the securities of our system against the honest-but-curious server (for the SNT system) with respect to data secrecy in Theorem \ref{sec_thm}; and against the collusions of trainers (for both SNT and FNT systems) with respect to data inversion given weight parameters in Theorems \ref{collusion_thm} and \ref{collusion_thm_rnt}.  

Below, we discuss a potential issue and additional hedges related to the SNT and FNT systems. 

{\medskip \noindent  \bf The potential issue.}  In our system each trainer sends out a set of trained weights. While exactly inverting the data from these weight parameters can be hard as shown above, the weights may potentially leak {\em some} information on the training data. 

{\medskip \noindent  \bf Known impossibility \cite{DN_impossibility2010}.} A \lq\lq perfect" notion of privacy, known as the \lq\lq Dalenius desideratum," states that the learning model (the neural network and its weight parameters in our case) should reveal no more about the input to which it is applied than would have been known about this input without applying the model. Unfortunately, it has been shown that this kind of \lq\lq perfect" privacy cannot be achieved by any useful learning model \cite{DN_impossibility2010}. As a related note, if all trainers in our systems do nothing but only transmit uniformly random weights, then the systems have perfect weight privacy, but no utility at all.

\medskip
\noindent
Given the above impossibility result, below we survey and discuss how to use orthogonal works in the literature to create hedges protecting the weight privacy in our systems. These hedges (either one or more) can be deployed locally at each trainer depending on the specific application at hand.

{\medskip \noindent  \bf Hedge 1: adding differential privacy.} It is possible to  make the shared weights differentially private,  by either adding Laplace or Gaussian noises to the local gradients as in \cite{AbadiCGMMT016, SongCS13, BassilyST14}, when  updating the weight parameters at each trainer; or using the classical technique of output perturbation as in \cite{0001LKCJN17}. When public non-labeled data is available, each trainer can also use the technique in \cite{Papernot17} locally. The strong composition theorem \cite{DRV2010} can be used when multiple differentially private weights are sent subsequently. Differential privacy will make the shared weights less dependent on any specific data item, thus making  each local dataset more protected. With this hedge, we need to accept a trade-off (e.g. as reported  in \cite{RahmanRLM18}) between testing accuracy and privacy, because noises are introduced into the training process. 

{\medskip \noindent  \bf Hedge 2: dropouts for privacy.} Dropout \cite{SrivastavaHKSS14} is the technique of randomly dropping neural nodes along with their connections from the neural network during training.  Dropout makes each node in the neural network  less depend on the training data, so intuitively having a similar effect as differential privacy as observed in \cite{0002KTW15}. Besides that, it has been known as a regularization method to prevent overfitting (see Hedge 3 below). In Section \ref{exp_res}, the trainers in our systems make extensive use of dropouts.

{\medskip \noindent  \bf Hedge 3: no overfitting.} Each trainer should care about whether the weight it sends out overfits its local data or not. This is because  overfitting may lead to attacks as reported in \cite{ShokriSSS17}. Overfitting can be prevented  by each trainer via locally monitoring the training and testing accuracies and, if necessary, employing e.g. regularization as suggested in \cite{ShokriSSS17} or dropouts as in Hedge 2. For example, in our  experiments on UCI's Breast Cancer, Banknote Authentication, Skin/NonSkin and MNIST datasets in Section \ref{exp_res} below, the testing accuracies are relatively close to the training ones (all are $\geq 98\%$), meaning no overfitting, so that the trainers are protected by this hedge.

{\medskip \noindent  \bf Hedge 4: anonymous transmission.} The origin of the weight parameters is not needed for the   collaborative training process of our systems. Indeed, each trainer only needs to know the weight parameters from {\em some} previous trainer. Therefore, it is advised to  hide the origin of the weight parameters of a previous trainer from the next trainer. While this can be generally accomplished by anonymous communication techniques (see e.g. \cite{DC08_survey} and the references therein), in the SNT system, the server can help by just keeping the origin secret. For this hedge to be applicable, we need the number of honest trainers larger than ($>$) 1.

{\medskip \noindent  \bf Hedge 5: locally preventive attack.} Finally, each  trainer {\em itself} is free to locally conduct any preventive attack before its weight transmission and deploy any corresponding dedicated mitigation if necessary to protect the weight. 

{\medskip \noindent  \bf Effects against concrete attacks.} First, consider  the following scenario. An honest-but-curious trainer $A$ receives a set of weight parameters from some previous trainer, and $A$ applies its own data on the received weight parameters $W$, with the goal of checking whether the data is in the training set producing $W$. Hedge 4 hides the origin of $W$, so in turn protecting the origin of the training data used to produce the weight parameters. As a consequence, the attacker cannot know {\em which} trainer owns the data. In addition, Hedges 1 and 2 help protecting $W$ in terms of differential privacy, so that $W$ does not depend heavily on any specific data item.

Second, consider the membership inference attack as in \cite{ShokriSSS17}. 
It is worth noting that outsiders (not the trainers) cannot mount the membership inference attack in our setting, because each trainer only uses its own data so that shadow training as in \cite{ShokriSSS17} is not possible. Only an insider (some trainer) can mount  the membership inference attack. Hedges 1, 2 and 3 have been already known in \cite{ShokriSSS17} as countermeasures for the attack. Hedge 4 has the same effect as above. Note that Hedges 1, 2, and 3 are concrete instantiations of Hedge 5. And as an additional example of Hedge 5, each trainer can locally deploy the technique of adversarial regularization proposed in \cite{NasrSH18}, so that the  predictions of the neural network model and its weight $W$ become hard to be exploited. 
}

\section{Experiments}\label{exp_res}
{\color{black} To show the general applicability of our systems, we conduct experiments on various datasets including UCI datasets, and image datasets (MNIST, CIFAR-10, and CIFAR-100).}

{\medskip \noindent \bf Environment.} For the UCI datasets, we employ a laptop and only run our programs on its CPU (Intel Core i7-2860QM, 2.5GHz). For the experiment with the MNIST dataset, we employ a machine with Intel(R) Xeon(R) CPU E5-2660 v3 @ 2.60GHz with Cuda-8.0 and GPU Tesla K40m; with Python 2.7.12 distributed in Anaconda 4.2.0 (64-bit). We assume a 1 Gbps channel between the trainers and the server (in Figure \ref{star_system}) and between the trainers (in Figure \ref{ring_system}).

{\medskip\noindent \bf Symmetric encryption.} For the SNT system, we use AES-128-CBC  encryption in OpenSSL which satisfies security against chosen plaintext attacks (CPA). The secret symmetric key has 128 bits. 

{\color{black}
{\medskip\noindent \bf Random partitions of the original dataset.} In below experiments for UCI, MNIST and CIFAR datasets, the trainers hold local training datasets which are distributed approximately identically, done via randomly shuffling and equally partitioning the original training dataset. In addition, for UCI datasets, the labels in the testing set are chosen with the ratio of classes roughly equivalent to the original dataset, done via using the {\tt train\_test\_split} function of {\tt sklearn}. More details are given below.

Equivalent distribution of classes is the case of  high interest  when trainers (oganizations) collaborate. Indeed, each trainer may have insufficient data (of a class) for good training, but should not have no data of that class at all. Taking the credit card fraud detection for example, if a credit card company has no  fraud transaction at all, then there is little reason why it has to collaborate.

}

\subsection{Experiments with UCI datasets}
Using our systems (both SNT and FNT), we conduct experiments on UCI datasets, and the results for the accuracy and F-score are given in Table \ref{Table_UCI_comparison}. Our accuracy and F-score are better than the results reported in \cite{AonoHPW16}. This is possibly due to the fact that our systems, while protecting data privacy, do not reduce the accuracy of the underlying learning algorithm (as proved in Theorems \ref{acc_thm} and \ref{acc_thm_rnt}). Below are details of each dataset and its corresponding experiment. In the following neural networks, the activation function used in the hidden nodes is the rectified linear unit (ReLU) and the sigmoid function is used at the output nodes. \textcolor{black}{The percentage of label 0 (or 1 if this percentage is higher) is reported in Table \ref{Table_UCI_comparison} approximately holds for all training and testing datasets.}

{\medskip\noindent\bf Pima (diabetes).} We use a neural network with the following architecture: 8 (input) - 512 - dropout (0.6) - 64 - dropout (0.4) - 1 (output). Namely, the 8 features (as real numbers) from each data item are fed to the input layer, which is connected to 512 hidden neural nodes. Next, the 512 hidden nodes are connected to 64 hidden nodes with a dropout rate of 0.6, and so forth. This neural network consists of 37,505 trainable parameters. To ensure the replicability  of  results, we fix the random seed of Python as {\tt random.seed(12345)}, the numpy random seed as {\tt numpy.random.seed(15)}, and the tensorflow seed as {\tt tensorflow.set\_random\_seed(1234)}. The original dataset \textcolor{black}{of  $768$ records (500 data items with label 0, and 268 with label 1)} is randomly split in a ratio of 8:2 for training \textcolor{black}{of  $614$ records and testing of   $154$ records (99 data items  with label 0, and 55 with label 1) via using the {\tt train\_test\_split} function of {\tt sklearn}}. Moreover, we envision the number of trainers to be $L = 20$ in our systems, so  the training set is split into 20 parts, each \textcolor{black}{(of at least 30 records)} held by a distributed trainer. The Adam optimizer is used with a fixed learning rate of 0.0002 and each trainer  uses a batch size of 128 data items. Each distributed trainer runs 150 local epochs on its own data before passing its encrypted weights to the server (or the next trainer). The number of central epochs is 20, meaning that each trainer  uploads and downloads the encrypted weights 20 times  from the server in SNT (or from the previous trainer in FNT). It takes less than 7 minutes for the system to finish its operations.

 \textcolor{black}{The  use of {\tt train\_test\_split} function of {\tt sklearn}  in the Pima dataset is the same for other UCI datasets below.}

{\medskip\noindent\bf Breast Cancer.} We use a neural network with the following architecture: 9 (input) - 32 - dropout (0.1) - 40 - dropout (0.2) - 64 - dropout (0.4) - 64 - dropout (0.4) - 64 - dropout (0.4) - 8 - 4 - 1 (output). This neural network consists of 14,785 trainable parameters. The random seeds are set as {\tt numpy.random.seed(15)},  {\tt random.seed(12345)}, and {\tt tensorflow.set\_random\_seed(1234)}. The original dataset is randomly split into a training set of 391 records and a testing set of 292 records. The training set is further split randomly into 20 parts, each \textcolor{black}{(of at least 19 records)} held by a distributed trainer. The Adam optimizer is used with a fixed learning rate of 0.0002 and each trainer uses a batch size of 128 data items. Each distributed trainer runs 40 local epochs on its own data before passing its encrypted weights to the server (or the next trainer). The number of central epochs is 5, meaning that each trainer uploads and downloads five times the encrypted weights from the server in SNT (or from the previous trainer in FNT). It takes less than 1 minute for the system to finish.

\begin{table*}[t]
\centering
\caption{Comparison with previous results of Aono et al. \cite{AonoHPW16} for encrypted data.}\label{Table_UCI_comparison}
\scalebox{0.8}{
\begin{tabular}{|c||c|c||>{\color{black}}c|c|c|}
\hline
\bf UCI   & \bf Known accuracy & \bf Known F-score& \bf Percentage of label 0 &\bf Our system accuracy & \bf Our system F-score\\
\bf Dataset Name &\multicolumn{2}{c||}{(in \cite{AonoHPW16}, data privacy is preserved)}&\bf (or label 1 if higher)&\multicolumn{2}{c|}{(data privacy is preserved)}\\
\hline
Pima (diabetes) & 80.70\% & 0.688525 & 64.29\%& \bf 85.06\%&\bf 0.763636\\
\hline
Breast Cancer &  98.20\%& 0.962406 &64.04\%&\bf 99.31\% &\bf 0.989304\\
\hline
Banknote Authentication & 98.40\% & 0.984615& 55.97\% &\bf 100.0\%& \bf 1.0\\
\hline
Adult Income &  81.97\% & 0.526921 & 76.38\% &\bf 85.90\%& \bf 0.664362\\
\hline
Skin/NonSkin & 93.89\%& 0.960130 &79.53\%&\bf 99.95\%&\bf 0.998655\\
\hline
\end{tabular}}
\end{table*}

{\medskip\noindent\bf Banknote Authentication.} We use a neural network with the following architecture: 4 (input) - 128 - dropout (0.7) - 64 - dropout (0.5) - 64 - dropout (0.5) - 1 (output). This neural network consists of 13,121 trainable parameters. The random seeds are set as {\tt numpy.random.seed(15)},  {\tt random.seed(12345)}, and {\tt tensorflow.set\_random\_seed(1234)}. We split the dataset into two parts: a training set with \textcolor{black}{786} records and a testing set with \textcolor{black}{586} records. The training set is further split randomly into 20 parts, each \textcolor{black}{(of at least 39 records)} held by a distributed trainer. The Adam optimizer is used with a fixed learning rate of 0.0002 and each trainer uses a batch size of 128 data items. Each distributed trainer runs 70 local epochs on its own data before passing its encrypted weight to the server (or the next trainer). The number of central epochs is 1, meaning that each trainer uploads and downloads 1 the encrypted weights once at the server in SNT (or from the previous trainer in FNT). It takes less than 11 seconds for the system to finish.

{\medskip\noindent\bf Adult Income.} We use a neural network with the following architecture: 14 (input) - 64 - dropout (0.4) - 32 - dropout (0.2) - 1 (output). This neural network consists of 3,073 trainable parameters. The random seeds are set as {\tt numpy.random.seed(4)} and {\tt tensorflow.set\_random\_seed(1234)}. The training set \textcolor{black}{of  $32561$ records} is split randomly into 20 parts, each \textcolor{black}{(of at least 1268 records)}  held by a distributed trainer. \textcolor{black}{The testing set is of  $16281$ records.} The Adam optimizer is used with a fixed learning rate of 0.0002 and each trainer uses a batch size of 128 data items. Each distributed trainer runs 14 local epochs on its own data before passing its encrypted weights to the server (or the next trainer). The number of central epochs is 120, meaning that each trainer uploads and downloads 120 times the encrypted weights from the server in SNT (or from the previous trainer in FNT). It takes around 100 minutes for the system to finish.

{\medskip\noindent\bf Skin/NonSkin.} We use a neural network with the following architecture: 3 (input) - 64 - dropout (0.4) - 32 - dropout (0.2) - 1 (output). This neural network consists of 2,369 trainable parameters. The random seeds are set as {\tt numpy.random.seed(15)},  {\tt tensorflow.set\_random\_seed(1234)}, and {\tt random.seed(12345)}. The original dataset \textcolor{black}{of  $245057$ records} is randomly split in a ratio of 8:2 for training \textcolor{black}{($196045$ records)} and testing \textcolor{black}{($49012$ records)}. The training set is further split randomly into 20 parts, each \textcolor{black}{of at least $9802$ records} held by a distributed trainer. The Adam optimizer is used with a fixed learning rate of 0.0002 and each trainer uses a batch size of 128 data items. Each distributed trainer runs 10 local epochs on its own data before passing its encrypted weight to the server (or the next trainer). The number of central epochs is 20, meaning each trainer uploads and downloads 20 times the encrypted weights from the server in SNT (or from the previous trainer in FNT). It takes less than 75 minutes for the system to finish.

{\medskip \noindent \bf Interlude: Credit Card Fraud Detection.} The dataset is provided in \cite{credit_kaggle} and consists of credit card transactions by European users over 2 days. There are 30 features in each transaction, including the time and amount of the transaction. We randomly split the original dataset in a ratio of \textcolor{black}{9:1 for training and testing. Precisely, the numbers of labels 0 and labels 1 in the original dataset  are 284,315 and 492 (i.e. 99.82725\% of 0s) respectively; in the training test those are 255,880 and 446 (i.e. 99.83\% of 0s); and in the testing set those are 28,435 and 46 (i.e. 99.83848\% of 0s) respectively}. Each transaction is labeled by 0 (normal) or 1 (fraud).  We envision the number of trainners to be $L =20$ in our systems, so the training set is further split into 20 parts each consisting of approximately \textcolor{black}{12,816} transactions. Each trainer employs a neural network with the following architecture: 30 (input) - 64 - dropout (0.4) - 32 - dropout (0.2) - 1, \textcolor{black}{with 4,097 parameters}. The Adam optimizer is used with a learning rate of 0.0002 and a batch size of 128. The seeds for randomness are set as {\tt numpy.random.seed(15)}, {\tt random.seed(12345)}, and {\tt tensorflow.set\_random\_seed(1234)}. Each trainer uses 10 local epochs before sending the weights out. The number of central epochs is 30. The running time of the system is around 65 minutes. The testing accuracy is around 99.96\% with a corresponding F-score of 0.8.

\subsection{Experiments with the MNIST dataset}
Below are experiments using an MLP and a CNN performed with the MNIST dataset.

{\medskip \noindent \bf Dataset.} We use the MNIST dataset \cite{MNIST} containing $28\times 28$ images of hand written numbers. We assume five trainers, namely $\nc = 5$ in Figures \ref{star_system} \ref{ring_system}, each of which holds a portion of 10,000 images  \textcolor{black}{uniformly and randomly split} from the \textcolor{black}{original} MNIST dataset. \textcolor{black}{We do that by randomly shuffling  the indexes of the training dataset and  splitting  it into 5 portions.} The validation set contains 10,000 images and the test set contains 10,000 images. Because the training set of each trainer is insufficient in size (relative to the validation set and the test set), the trainers wish to learn from the combined training set with a a total of $5\cdot 10,000 = 50,000$ images).

{\medskip \noindent \bf Calculation of running time.} On the basis of Figures \ref{star_system} and \ref{ring_system}, the running time $\bfT_{\sm our\ system}$ of our system can be expressed as 
\begin{eqnarray}
\bfT_{\sm our\ system}\!&=&\! \nep\sum_{i =1}^{\nc} \left(\bfT_{\sm original \ SGD}^{(i)} + \bfT_{\up}^{(i)} +  \bfT_{\down}^{(i)} + \bfT_{\sm enc}^{(i)} + \bfT_{\sm dec}^{(i)}\right)  \label{our_system_run_time}
\end{eqnarray}
in which $\nep$ is the number of epochs, where an epoch is a loop over all data items in the combined dataset. For trainer $i$, $\bfT_{\sm original \ SGD}^{(i)}$ is the running time of the original SGD over the local dataset of the trainer,  $\bfT_{\up}^{(i)}$ is the time to upload a ciphertext to the server (or the next trainer);  $\bfT_{\down}^{(i)}$ is the time to download a ciphertext from the server (or the previous trainer), $\bfT_{\sm enc}^{(i)}$ is the time for one symmetric encryption, and $\bfT_{\sm dec}^{(i)}$ is the time for one symmetric decryption.

\begin{figure*}[t]
\centering
\begin{tabular}{cc}
\includegraphics[scale=0.4]{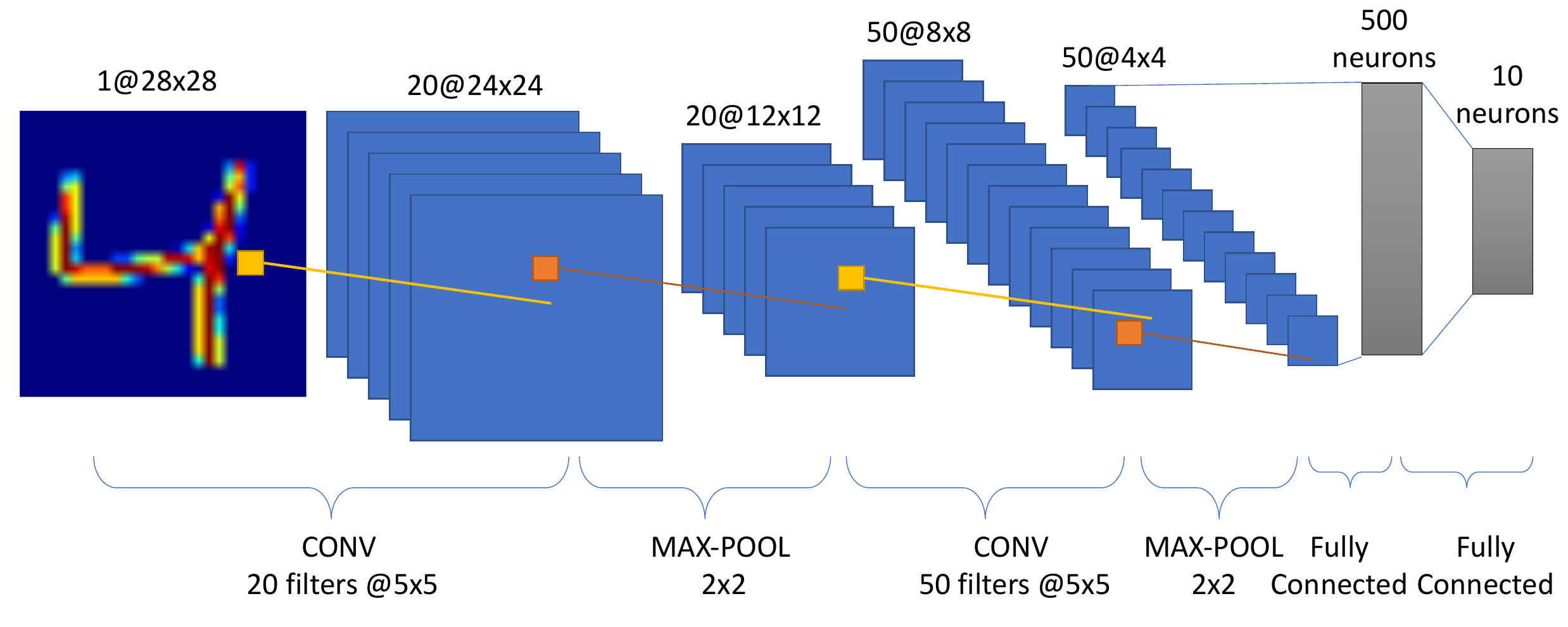} & \includegraphics[scale=0.4]{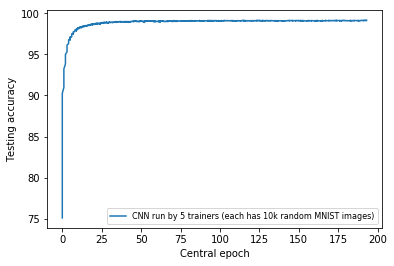}
\end{tabular}
\caption{\color{black} Convolutional neural network of LeNet-type \cite{Lecun98,dl_tutor} ({\bf left}) and its testing accuracy graph produced by 5  trainers in our systems ({\bf right}).}\label{lenet_fig}
\end{figure*}

\begin{table}[t]
\centering 
\caption{Timings in experiments with MNIST.}\label{MNIST_experiment}
\begin{tabular}{cccccc}
\hline
\bf Network model &$\bfT_{\sm original \ SGD}^{(i)}$ & $\bfT_{\up}^{(i)}$ & $\bfT_{\down}^{(i)}$  &  $\bfT_{\sm enc}^{(i)}$ & $\bfT_{\sm dec}^{(i)}$\\
\hline
\bf MLP \cite{dl_tutor}& 0.3 (min.)& 0.05 (sec.)& 0.05 (sec.)& 0.12 (sec.)& 0.06 (sec.)\\
\bf CNN \cite{dl_tutor}  & 3.3 (sec.)& 0.051 (sec.)& 0.051 (sec.)& 0.13 (sec.)& 0.06 (sec.)\\
\hline
\end{tabular}
\end{table}

{\medskip \noindent\bf Using an MLP network.} 
In this experiment, each trainer runs the MLP code in \cite{dl_tutor} with a batch of one data item over its dataset of 10,000 images. The number of hidden nodes is 500. The number of weight parameters is identical to the number of gradient parameters, which is 397,510. We use Python's package {\tt pickle} to convert these weight parameters as Python objects into a byte stream for encryption, resulting in a plaintext with a size of around 4.33 MB. The corresponding ciphertext is also pickled, having a size of around 5.78 MB, which is sent via our 1 Gbps network in 0.05 seconds. Table \ref{MNIST_experiment}  gives the approximate timings of trainers $1\leq i \leq 5$.

\medskip
\noindent
In our experiment, at epoch number 639,  Client 1 obtains the best validation score of 1.63\% with a test performance of 1.64\%. Therefore, the accuracy over the validation set is $100 - 1.63 = 98.37\%$ and the accuracy over the testing set is $100 - 1.64 = 98.36\%$. The running time of the system to obtain the result, theoretically estimated via (\ref{our_system_run_time}), is
\begin{eqnarray*}
\lefteqn{\bfT_{\sm our\ system}}\\ 
&=& \nep\sum_{i =1}^{\nc} \left(\bfT_{\sm original \ SGD}^{(i)} + \bfT_{\up}^{(i)} + \bfT_{\down}^{(i)}  +\  \bfT_{\sm enc}^{(i)} + \bfT_{\sm dec}^{(i)}\right)\\
&=& 639\cdot 5 \cdot \left(0.3 + \frac{0.05}{60} + \frac{0.05}{60} + \frac{0.12}{60} + \frac{0.06}{60}\right)\\
&\approx& 973.41 \mbox{ (minutes)} \\
&\approx& 16.22 \mbox{ (hours)}.
\end{eqnarray*}
The running time of the original code for SGD over the combined dataset of $5\cdot 10^4$ images is around 32 seconds per epoch. Thus, when $\nep =639$, we have  $\bfT_{\sm original \ SGD} = 32\cdot 639 \ {\rm (seconds)} = 5.68 \ {\rm (hours)}$. Therefore,  $$\frac{\bfT_{\sm our\ system}}{\bfT_{\sm original \ SGD}} = \frac{16.22}{5.68} \approx 2.86 < 3,$$
which supports the  claim of experimental efficiency in Section \ref{our_contri}. 


{\medskip \noindent\bf Using a CNN.}
We make use of a CNN (LeNet-type \cite{Lecun98,dl_tutor}) as depicted in Figure \ref{lenet_fig}. Each trainer runs the LeNet-type code in \cite{dl_tutor} with a batch of 500 data items over its dataset of 10,000 images. The number of weight parameters is 431,080, which is the sum of $20\times 5\times 5 + 20$ (first convolution layer), $50\times 20 \times 5\times 5+50$ (second convolution layer), $50\times 4 \times 4 \times 500 + 500$ (fully connected), and $500\times 10 + 10$. We use {\tt pickle} to convert these weight parameters as Python objects into a byte stream for encryption, resulting in a plaintext with a size of around 4.73 MB. The corresponding ciphertext is also pickled, having a size of around 6.31 MB, which is sent via our 1 Gbps network in around 0.051 seconds. Table \ref{MNIST_experiment} gives the approximate timings of trainers $1\leq i \leq 5$.

In our experiment, at epoch number \textcolor{black}{193},  the accuracy over the validation set is $99.1\%$ and the accuracy over the testing set is $99.17\%$. The running time of the system to obtain the result, theoretically estimated via (\ref{our_system_run_time}), is
\begin{eqnarray*}
\lefteqn{\bfT_{\sm our\ system}}\\ &=& \nep\sum_{i =1}^{\nc} \left(\bfT_{\sm original \ SGD}^{(i)} + \bfT_{\up}^{(i)} + \bfT_{\down}^{(i)}  + \ \bfT_{\sm enc}^{(i)} + \bfT_{\sm dec}^{(i)}\right)\\
&=& {\color{black}193}\cdot 5 \cdot (3.3 + 0.051 + 0.051 + 0.13 + 0.06)\\
&\approx& 3466 \mbox{ (seconds)} \\
&\approx& 0.96 \mbox{ (hours)}.
\end{eqnarray*}
Thus, 
\begin{eqnarray*}
\frac{\bfT_{\sm our\ system}}{\bfT_{\sm original \ SGD}} \approx \frac{0.96}{0.89} < 1.1,
\end{eqnarray*}
 supporting the  claim of experimental efficiency in Section \ref{our_contri}.

\subsection{\textcolor{black}{Experiments with the CIFAR-10 and CIFAR-100 datasets}}
{\color{black}
The CIFAR-10 dataset \cite{cifar10_dataset} consists of 50,000 RGB images of shape $32\times 32\times 3$ for training, and 10,000 RGB images of the same shape for testing.

The CIFAR-100 dataset is just like the CIFAR-10, except it has 100 classes (of dolphin, cockroach, motorcycle etc.) containing 600 images each. There are 500 training images and 100 testing images per class, so that the training dataset and testing dataset have respectively 50,000 and 10,000 images  as above.

%
%

\begin{table}[t]
\caption{\color{black} Timings in experiments with CIFAR-10 and CIFAR-100.}\label{CNN_CIFAR10}
\centering
\begin{tabular}{>{\color{black}}c >{\color{black}}c >{\color{black}}c >{\color{black}}c >{\color{black}}c >{\color{black}}c}
\hline
\bf Network model, Dataset & $\bfT_{\sm original}^{(i)}$ & $\bfT_{\up}^{(i)}$ & $\bfT_{\down}^{(i)}$  &  $\bfT_{\sm enc}^{(i)}$ & $\bfT_{\sm dec}^{(i)}$\\
 \hline
\ \ \ \ \ \ \ \ \ \bf CNN, CIFAR-10 &\ \ 37 (sec.)& 0.056 (sec.)& 0.056 (sec.)& 0.130 (sec.)& 0.06 (sec.)\\
\bf ResNet-20v1, CIFAR-10 &\ \ 59  (sec.)& 0.013 (sec.)& 0.013 (sec.)& 0.064 (sec.)& 0.06 (sec.)\\
\bf ResNet-32v1, CIFAR-10 &\ \ 87 (sec.)& 0.023 (sec.)& 0.023 (sec.)& 0.066 (sec.)& 0.06 (sec.)\\
\hline
\ \bf ResNet-44v1, CIFAR-100 & 122 (sec.)& 0.032 (sec.)& 0.032 (sec.)& 0.091 (sec.)& 0.085 (sec.)\\
\bf ResNet-110v2, CIFAR-100 & 322 (sec.)&  0.165 (sec.)&  0.165 (sec.)& 0.179 (sec.)& 0.138 (sec.)\\
\hline
\end{tabular}

{\color{black} \footnotesize $^*$In the table, the time for training over plain local data $\bfT_{\sm original}^{(i)}$ dominates the others.}
\end{table}

We assume there are 5 trainers, each of which holds a portion of 10,000 images randomly and uniformly split  from the original 50,000 training ones. Therefore, the local training datasets of the trainers have approximately  identical distribution of labels.

Each trainer in our system re-uses the Keras official example codes \cite{keras_cifar10_cnn} (of CNN) and \cite{keras_cifar10_resnet} (of ResNetv1 \cite{HeZRS16} and ResNetv2 \cite{HeZRS16v2}) for training and testing over the CIFAR-10 and CIFAR-100 datasets. The running times of each trainer is given at Table \ref{CNN_CIFAR10} in which the time for local training over plain data dominates the others of cryptographic operations and transmissions, i.e. 
$$\bfT_{\sm original}^{(i)} \gg \bfT_{\up}^{(i)} + \bfT_{\down}^{(i)} + \bfT_{\sm enc}^{(i)} + \bfT_{\sm dec}^{(i)}.$$
 Therefore, the running time of our system is approximate that of the original one where all data items are centralized, namely,
\begin{eqnarray*}
\bfT_{\sm our\ system} \approx \bfT_{\sm original}  
\end{eqnarray*}
in the experiments. More details are given below.

\medskip
\noindent
\underline {\bf Experiments with CIFAR-10}

{\medskip \noindent \bf Using the CNN of \cite{keras_cifar10_cnn}.} The number of weight parameters is 1,250,858 with following network architecture: CONV -- ReLU -- CONV -- ReLU -- MaxPooling2D -- Dropout(0.25) -- CONV -- ReLU -- CONV -- ReLU -- MaxPooling2D -- Dropout(0.25) -- Flatten -- Dense(512) -- ReLU -- Dropout(0.5) -- Dense(10) -- softmax. It is reported in \cite{keras_cifar10_cnn} that the network reaches testing accuracy of  around 79\% after 50 epochs.

In our experiment, we consider 5 trainers, each uses the same network as above. The local number of epochs at each trainer is 5, and the central number of epoch is 10, so that each data item is used $5\times 10 = 50$ times (same as the original \cite{keras_cifar10_cnn}, of 50 epochs as described above). The ciphertext encrypting the weight parameters seen as a numpy array is of 6.9 MB, which is sent via our 1 Gbps network in around 0.056 seconds. The time for symmetric encryption and decryption is given in Table \ref{CNN_CIFAR10} and it is worth noting that the local training time at each trainer $\bfT_{\sm original}^{(i)}$ dominates other timings.

\begin{figure*}[t]
\centering
\begin{tabular}{cc}
\color{black} \footnotesize (a) Experiments with CIFAR-10 & \color{black} \footnotesize (b) Experiments  with CIFAR-100\\
\includegraphics[scale=0.5]{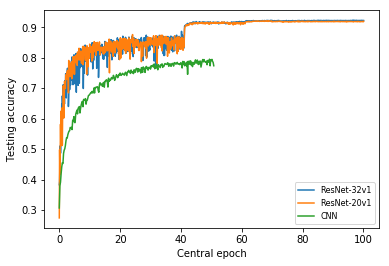} & \includegraphics[scale=0.5]{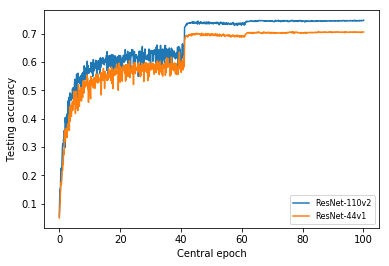} 
\end{tabular}
\caption{\color{black} Testing accuracy graphs of our systems (5 trainers each having 10,000  training images) with various neural network architectures.}\label{testing_acc_cifar10}
\end{figure*}

The testing accuracy graph obtained by the above execution of our system is depicted in Figure \ref{testing_acc_cifar10}(a) in which the testing accuracy per central epoch of each trainer is given. The maximum of testing accuracy of $79.5\%$ is obtained at epoch 49, complying with the original result  in \cite{keras_cifar10_cnn}.

{\medskip \noindent \bf Using ResNet-20v1.} The number of weight parameters is 274,442 and the Keras code is at \cite{keras_cifar10_resnet}. Each of the 5 trainers runs 2 local epochs before sending the weight out. The number of central epochs is 100, so that each data item is used $2\times 100= 200$ as in the original code \cite{keras_cifar10_resnet}. The ciphertext encrypting the weight parameters seen as a numpy array  is of 1.6 MB, and the symmetric encryption and decryption times are given in Table \ref{CNN_CIFAR10}.

Each trainer in our experiment schedules the learning rate as follows: $10^{-3}$, $10^{-4}$, $10^{-5}$, $10^{-6}$, $10^{-6}/2$ if the central epoch is correspondingly in the ranges $[0,40]$, $[41, 60]$, $[61, 80]$, $[81, 90]$, $[91,99]$. These ranges are half of the ranges in  \cite{keras_cifar10_resnet} because we use 2 as the number of local epochs at each trainer.

The testing accuracy graph obtained by the above execution of our system with ResNet-20 is depicted in Figure \ref{testing_acc_cifar10}(a) in which the testing accuracy per central epoch of each trainer is given. The maximum of testing accuracy of $92.13\%$, complying with the result reported in \cite{keras_cifar10_resnet}, is obtained at central epoch 68.

{\medskip \noindent \bf Using ResNet-32v1.} The number of weight parameters is 470,218 and the Keras code is at \cite{keras_cifar10_resnet}. The ciphertext encrypting the weight parameters seen as a numpy array  is of 2.8 MB, and the symmetric encryption and decryption times are given in Table \ref{CNN_CIFAR10}. The number of local epochs, central epochs and the learning rate schedule at each trainer are the same as ResNet-20. 

The testing accuracy graph obtained by the above execution of our system with ResNet-32 is depicted in Figure \ref{testing_acc_cifar10}(a) in which the testing accuracy per central epoch of each trainer is given. The maximum of testing accuracy of $92.28\%$, complying with the result reported in \cite{keras_cifar10_resnet}, is obtained at central epoch 97.

\medskip
\noindent
\underline {\bf Experiments with CIFAR-100}

%

{\medskip \noindent \bf Using ResNet-44v1.} The number of weight parameters is 671,844 and the Keras code is at \cite{keras_cifar10_resnet}. The ciphertext encrypting the weight parameters seen as a numpy array  is of 3.9 MB, and the symmetric encryption and decryption times are given in Table \ref{CNN_CIFAR10}. The number of local epochs, central epochs and the learning rate schedule at each trainer are the same as above. 

The testing accuracy graph obtained by the above execution of our system with ResNet-44 is depicted in Figure \ref{testing_acc_cifar10}(b) in which the testing accuracy per central epoch of each trainer is given. The maximum of testing accuracy of $70.69\%$, complying with the result reported in \cite{cifar100_resnet_tf}, is obtained at central epoch 91.

{\medskip \noindent \bf Using ResNet-110v2.} The number of weight parameters is 3,346,340 and the Keras code is at \cite{keras_cifar10_resnet}. The ciphertext encrypting the weight parameters seen as a numpy array  is of 20 MB, and the symmetric encryption and decryption times are given in Table \ref{CNN_CIFAR10}. The number of local epochs, central epochs and the learning rate schedule at each trainer are the same as above. 

The testing accuracy graph obtained by the above execution of our system with ResNet-110v2 is depicted in Figure \ref{testing_acc_cifar10}(b) in which the testing accuracy per central epoch of each trainer is given. The maximum of testing accuracy of $74.75\%$, complying with the result reported in \cite{cifar100_resnet_tf}, is obtained at central epoch 72.

%

}

\section{Conclusion}
We have conctructed privacy-preserving systems in which multiple machine-learning trainers can use SGD or its variants over the combined dataset of all trainers without having to share the local dataset of  each trainer. Differing from previous works, our systems make use of the weight parameters rather than the gradient parameters. The experimental results show that our systems are practically efficient in terms of computation and communication, and effective in terms of accuracy. 

\section*{Acknowledgment}
The first author is partially funded by JST CREST JPMJCR168A. 

\end{document}